\documentclass[conference]{IEEEtran}
\IEEEoverridecommandlockouts
\ifCLASSINFOpdf
\else
\fi

\usepackage{subfigure}
\usepackage{amsmath}
\usepackage{amsfonts}

\usepackage{tikz}
\usetikzlibrary{arrows}

\newtheorem{theorem}{Theorem}[section]
\newtheorem{proposition}{Proposition}[section]

\newtheorem{example}{Example}[section]

\def\N{\mathbb{N}}

\def\R{\mathbb{R}}

\def\G{\mathcal{G}}

\def\e{\varepsilon}

\def\1{\mathds{1}}

\def\m{\mathrm{m}}
\def\ss{\mathrm{ss}}

\def\EW{\mathrm{E}_\mathrm{Wards}}
\def\sWards{{\sc sWards}}

\def\sCEC{{\sc sCEC}}
\def\CEC{{\sc CEC}}
\def\ESW{\mathrm{E}_\mathrm{sWards}}

\def\tr{\mathrm{tr}}

\newcommand\D[1]{D\langle #1 \rangle}

\usepackage{color}

\begin{document}
\title{Spherical Wards clustering and generalized Voronoi diagrams}


\author{\IEEEauthorblockN{Marek \'Smieja}
\IEEEauthorblockA{Faculty of Mathematics and Computer Science\\Jagiellonian University\\
            Lojasiewicza 6, 30-348 Krakow, Poland\\
Email: marek.smieja@ii.uj.edu.pl}
\and
\IEEEauthorblockN{Jacek Tabor}
\IEEEauthorblockA{Faculty of Mathematics and Computer Science\\Jagiellonian University\\
            Lojasiewicza 6, 30-348 Krakow, Poland\\
Email: jacek.tabor@ii.uj.edu.pl}}

%\author{\IEEEauthorblockN{Author 1}
%\IEEEauthorblockA{Affiliation 1}
%\and
%\IEEEauthorblockN{Author 2}
%\IEEEauthorblockA{Affiliation 2}}


%Telephone: (800) 555--1212, Fax: (888) 555--1212}


\IEEEoverridecommandlockouts
\IEEEpubid{\makebox[\columnwidth]{Copyright notice: 978-1-4673-8273-1/15/\$31.00~\copyright2015 IEEE \hfill} \hspace{\columnsep}\makebox[\columnwidth]{ }}

%\IEEEpubid{\makebox[\columnwidth]{\hfill 9781-4244-3941-6/09/\$25.00~\copyright~2009 IEEE}}
%\hspace{\columnsep}\makebox[\columnwidth]{Published by the IEEE Computer Society}}

\maketitle

\begin{abstract}
Gaussian mixture model is very useful in many practical problems. Nevertheless, it cannot be directly generalized to non Euclidean spaces. To overcome this problem we present a spherical Gaussian-based clustering approach for partitioning data sets with respect to arbitrary dissimilarity measure. The proposed method is a combination of spherical Cross-Entropy Clustering with a generalized Wards approach. The algorithm finds the optimal number of clusters by automatically removing groups which carry no information. Moreover, it is scale invariant and allows for forming of spherically-shaped clusters of arbitrary sizes. In order to graphically represent and interpret the results the notion of Voronoi diagram was generalized to non Euclidean spaces and applied for introduced clustering method.
\end{abstract}


% \fi
\IEEEpeerreviewmaketitle

\section{Introduction}

Distribution-based clustering, such as Gaussian mixture model (GMM), has been proven to be very useful in many practical problems \cite{mclachlan2007algorithm}. This technique has been widely applied in object detection \cite{figueiredo2002unsupervised}, learning and modeling \cite{samuelsson2004waveform}, feature selection \cite{xing2001feature} or classification \cite{povinelli2004time}. The constructed groups are described by optimally fitted probability distributions. Nevertheless, this kind of methods is limited for the case of Euclidean spaces and the clustering of data with respect to Gaussian-like probability distributions in arbitrary data spaces where only distance or (dis)similarity measure is provided still remains a challenge. 

In this paper we show how to partially overcome this problem and propose a spherical Wards clustering ({\sWards}) which divides data sets with respect to arbitrary dissimilarity measure into groups described by spherical Gaussian-like distributions. Figure \ref{fig:diag} shows the relationship between {\sWards} and related methods. Moreover, we extend the notion of Voronoi diagram to the case of arbitrary criterion function in non Euclidean spaces and apply it for {\sWards} clustering. 

Introduced method permits an informal interpretation of the notion of spherical Gaussian probability distribution in non Euclidean spaces. The algorithm is capable of discovering spherically-shaped groups of arbitrary sizes (see Example \ref{unbalancedData}). Moreover the clustering results are invariant with respect to the scaling of data (see Example \ref{scaleInvariance}). In fact, data sets with unbalanced groups appear very often in practice, e.g in chemoinformatics where finding of chemical compounds acting on specific disease is rare \cite{gasteiger2003handbook, dixon1999investigation} or in Natural Language Processing where the numbers of documents that belong to particular domains are different \cite{zamir1998web}. Our method can be successfully applied in discovering of populations districts in biological systems modeled by a random walk procedure (see Examples \ref{biloEx1}, \ref{biloEx2}). The method is easy to implement and has the same numerical complexity as the k-means version adapted to non Euclidean spaces \cite{batagelj1988generalized}. Moreover, our algorithm automatically finds the resultant number of groups by reducing unnecessary clusters on-line. Voronoi diagrams for {\sWards}, k-means and their kernelized versions for a mouse-like set with non Euclidean distance function are presented in Figure~\ref{fig:mouse}.

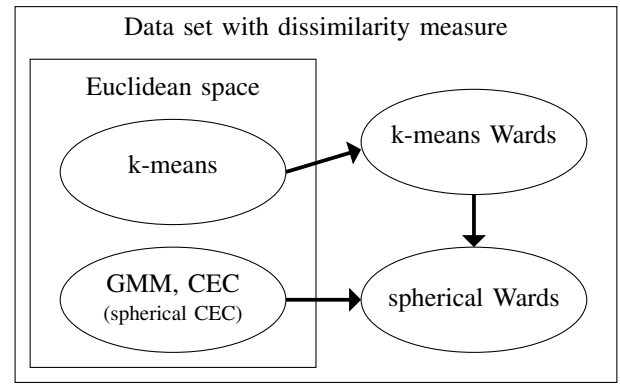
\begin{figure}[t]
\centering
\begin{tikzpicture} 
\node(title) at (4,4.7){Data set with dissimilarity measure};
\node(subtitle) at (2.1,3.9){Euclidean space};
\draw (0,0) rectangle (8,5);
\draw (0.2,0.2) rectangle (4,4.3);
\node(km) at (2.1,2.9){k-means};
\node(GMM) at (2.1,1.3){GMM, CEC};
\node(GMM) at (2.1,0.9){\footnotesize (spherical CEC)};
\draw (3.6,2.8) arc (0:360:1.5cm and 0.7cm);
\draw (3.6,1.1) arc (0:360:1.5cm and 0.7cm);
\node(wards) at (6.1,3.3){k-means Wards };
\node(swards) at (6.1,1.1){spherical Wards};
\draw (7.6,3.2) arc (0:360:1.5cm and 0.7cm);
\draw (7.6,1.1) arc (0:360:1.5cm and 0.7cm);
%\draw [->] (3.6,2.8) -- ++(1.0,0.3);
%\draw [->] (3.6,1.1) -- ++(1.0,0);
%\draw [->] (6.1,2.5) -- ++(0,-0.7);
    \draw[
        -triangle 90,
        line width=0.2mm,
        postaction={draw, line width=0.05cm, shorten >=0.1cm, -}
    ] (3.6,2.8) -- (4.6,3.1);
    \draw[
        -triangle 90,
        line width=0.2mm,
        postaction={draw, line width=0.05cm, shorten >=0.1cm, -}
    ] (3.6,1.1) -- (4.6,1.1);
    \draw[
        -triangle 90,
        line width=0.2mm,
        postaction={draw, line width=0.05cm, shorten >=0.1cm, -}
    ] (6.1,2.5) -- (6.1,1.8);
\end{tikzpicture}
	\caption{{\bf Spherical Wards clustering.} The relationship of our method with GMM, k-means and Wards clustering.}
	\label{fig:diag} 
\end{figure}

\begin{figure*}[t]
\centering
\subfigure[Wards k-means clustering with $k=4$.]{\includegraphics[width=1.65in]{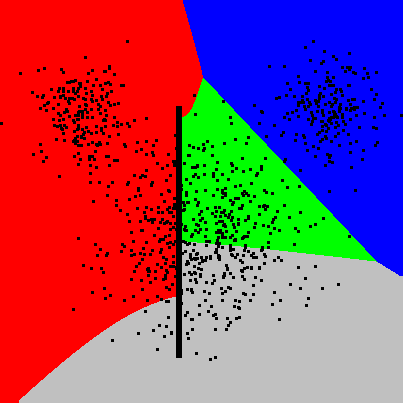}} \quad
\subfigure[{\sWards} clustering started with 10 initial clusters.]{\includegraphics[width=1.65in]{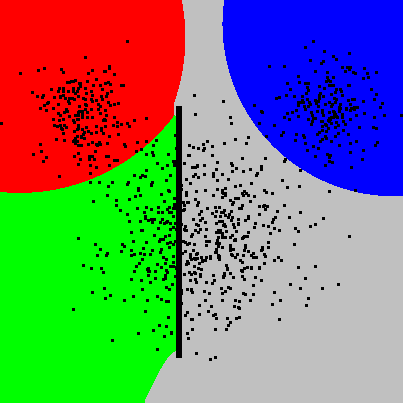}} \quad
\subfigure[Wards k-means with RBF dissimilarity function and $k=4$.]{\includegraphics[width=1.65in]{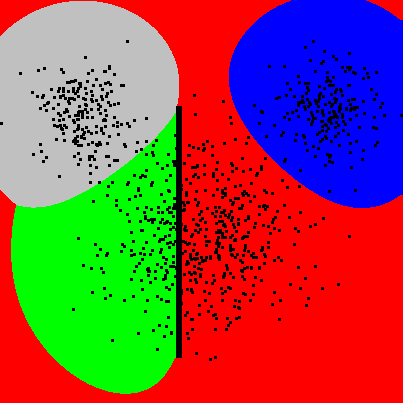}} \quad
\subfigure[{\sWards} clustering with RBF dissimilarity function started with 10 initial clusters.]{\includegraphics[width=1.65in]{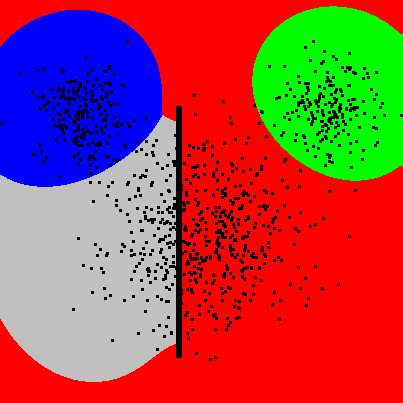}} \quad
	\caption{{\bf Voronoi diagrams.} The Voronoi diagrams of introduced {\sWards} compared with Wards k-means on mouse-like set with barriers with two types of similarity measures: Euclidean and RBF similarity. The barrier changes the distance between elements. The distance between elements located on the opposite sides of barrier is calculated as a length of the shortest path which does not cross the barrier.  Observe that despite the barrier {\sWards} method discovered ``mouse ears'' as a spherical clusters while ``mouse head'' was divided into two smaller groups. The Wards k-means results do not have so intuitive explanation. Kernelized versions of both algorithms gave the satisfactory effects, but the main difficulty lies in finding the appropriate values of RBF parameter. An important thing is that the introduced {\sWards} technique produces comparable partition without the need of parameters tuning.}
	\label{fig:mouse} 
\end{figure*}

Proposed {\sWards} method is a combination of spherical variant of Cross-Entropy Clustering ({\CEC}) \cite{tabor2013} with the generalized Wards approach \cite{batagelj1988generalized, spath1975cluster}. Generally, spherical {\CEC} describes clusters by optimally fitted spherical Gaussian distributions while Wards method allows for its adaptation to non Euclidean case. Spherical {\CEC} performs a clustering by optimizing a cross-entropy criterion function \eqref{spher1}. Its form is very flexible since it is based on the within clusters sums of squares, the cardinalities of clusters and the dimension of space. 

\IEEEpubidadjcol

Applied Wards approach allows for a generalization of the notion of within cluster sum of squares for the case of any dissimilarity measure \cite{batagelj1988generalized, spath1975cluster}. The key lies in the observation that this quantity can be rewritten in Euclidean space without the use of a mean $\m_Y$ of a cluster $Y$ in the form:
$$
\sum_{y \in Y}d^2(y,\m_Y) =\frac{1}{2|Y|}\sum_{y, z \in Y}d^2(y,z). 
$$
On the other hand, note that a dimension in arbitrary space does not have to be defined. Therefore, to adapt spherical CEC criterion function to general case we recommend to estimate its value from data with use of Maximum Likelihood Estimator of intrinsic dimension \cite{maxdim, comments}. 

To graphically represent and interpret the results of clustering the notion of Voronoi diagram is widely applied. Its construction requires the answer for the question: to which cluster we should associate an arbitrary unclustered point? In the case of classical k-means the answer is simple: we assign the point to the cluster with the nearest center. In the Wards method we replace it by a generalization of distance of point $x$ from the center of cluster $Y$ given by \cite{batagelj1988generalized}
\begin{equation} \label{voronoi}
d^2(x;Y):=\frac{1}{|Y|}\sum_{y \in Y}d^2(x,y)-\frac{1}{|Y|}\ss(Y).
\end{equation}
In our work we calculate the analogue of above formula \eqref{voronoi} for the case of {\sWards} criterion function \eqref{sWards11} (see \eqref{vorInt} for precise formula and Figure \ref{fig:mouse} for sample effects). 

The practical properties of proposed method are illustrated and examined on synthetic data sets and examples retrieved form the UCI repository \cite{asuncion2007uci}. We compare {\sWards} with similar methods which can be applied for non Euclidean data as k-means, Spectral Clustering and their kernelized versions. Our tests demonstrate that introduced method can be applied for populations detection in simple biological systems.

The paper is organized as follows. Next section gives a brief description of related clustering methods. In section 3 we recall Wards approach to k-means and present its application for spherical {\CEC} criterion function. Section 4 demonstrates the generalization of Voronoi diagrams to the case of arbitrary criterion functions in non Euclidean data paying particular attention on {\sWards} method. The results of experiments and potential applications are given in section 5 while section 6 contains the conclusion.

\section{Related works}

The hierarchical clustering is probably one of the most popular methods to partition data based on any kind of (dis)similarity measure \cite{johnson1967hierarchical}. The well-known k-means algorithm \cite{hartigan1979algorithm} can also be adapted to non Euclidean data by defining a medoid \cite{park2009simple} which plays a role of a generalized notion of mean or by using the Wards method \cite{batagelj1988generalized, spath1975cluster} which reformulates the within cluster sum of squares without the notion of the cluster mean. Despite the wide use of these methods, they are sometimes unable to discover groups with complex structures and different sizes. A lot of modifications were also considered to describe clusters with arbitrary shapes \cite{dhillon2004kernel, wagstaff2001constrained}.  Spectral Clustering uses eigenvectors of similarity matrix to divide elements into groups \cite{zhou2007spectral}. 

Another issue of clustering non Euclidean data sets is the appropriate selection of dissimilarity measure. Examples showed that interesting effects can be obtained by applying Gaussian radial basis function (RBF) \cite{chen1993clustering}. The difficulty is that there is no unified methodology how to choose the radius of this function for particular situation \cite{orr1995regularization, rippa1999algorithm}.

In order to perform a distribution-based clustering a GMM is widely used in Euclidean space \cite{mclachlan2007algorithm}. Nevertheless it cannot be directly generalized to arbitrary data sets with dissimilarity measures. On the other hand, a family of density based clustering such as DBSCAN \cite{kriegel2011density} can be applied for non Euclidean data. Although the method is capable of discovering clusters of arbitrary shapes and does not require the specification of the number of groups, it does not adopt well to clusters with large differences in densities.

Proposed {\sWards} method joins the simplicity and flexibility of k-means with the effects of GMM. Its can be applied in non Euclidean spaces and is based on Gaussian-like probability distributions.

\section{Clustering method}

The proposed {\sWards} clustering is a combination of spherical Cross-Entropy Clustering ({\sCEC}) \cite{tabor2013} with a generalized Wards approach \cite{batagelj1988generalized, spath1975cluster}. In this section we first introduce a basic notation and recall the Wards version of k-means. Then, we show how {\sCEC} can be generalized to non Euclidean data sets via Wards method.

\subsection{Wards method}

Generally, k-means method aims at producing a splitting of data set which optimizes a squared error criterion function. For a group $Y \subset \R^N$ the within cluster sum of squares is defined as:
$$
\ss(Y) = \sum_{y \in Y}\|y-\m_{Y}\|^2,
$$
where $\m_Y$ is a mean of $Y$. The k-means looks for a partition of $X \subset \R^N$ into $k$ pairwise disjoint sets $Y_1,\ldots,Y_k$  such that the function
$$
\sum_{j=1}^k \ss(Y_j)
$$
is minimal.

Note that the above formulas cannot be used directly for non vector data since the mean is not well-defined for general data sets. There are several alternatives \cite{jain1999,steinley2006k,gan2007data,jain2010,xu2009clustering} which allow to partially overcome this difficulty as k-medoids \cite{kaufman1987clustering} or k-clustering \cite{indyk1999sublinear}. The technique related to k-clustering and k-means is the generalized Wards method \cite{batagelj1988generalized, spath1975cluster} which plays the basic role in our investigations. The key idea is the observation that the within cluster sum of squares in Euclidean space can be formulated equivalently without the notion of the center of cluster:

\begin{proposition} \label{th:1} \cite{spath1975cluster} 
If $Y \subset \R^N$, then
$$
\begin{array}{c}
\sum \limits_{y \in Y} \! \|y-\m_Y\|^2\, =\frac{1}{2|Y|}
\sum \limits_{y \in Y} \sum \limits_{z \in Y} \! \|y-z\|^2\, ,
\end{array}
$$
where $|Y|$ is a cardinality of $Y$.
\end{proposition}

This allows to reasonably generalize the within cluster sum of squares to general non Euclidean data set. For this purpose let $X$ be an arbitrary data set and let $d~:~X~\times~X~\to~[~0~,~+~\infty~)$ be a symmetric dissimilarity measure on $X$, i.e,
\begin{itemize}
\item $d(y,y) = 0$, 
\item $d(y,z) = d(z,y)$,
\end{itemize}
for $y,z \in X$. Given two subsets $Y, Z$ of $X$ we define a function \cite{indyk1999sublinear} connected with the average linkage function (also called average neighbor function) \cite{sokal1958, gan2007data} as:
$$
\begin{array}{c}
\D{Y,Z}:=\sum \limits_{y \in Y} \sum \limits_{z \in Z} \! d^2(y,z) \, .
\end{array}
$$

As a generalized within cluster $Y \subset X$ sum of squares we put \cite{batagelj1988generalized}:
\begin{equation} \label{defWard}
\begin{array}{c}
\ss(Y):=\frac{1}{2|Y|}\D{Y,Y}=\frac{1}{2|Y|}
\sum \limits_{y \in Y} \sum \limits_{z \in Y} \! d^2(y,z)\, .
\end{array}
\end{equation}
Then, the goal of Wards method is formulated as follows: 

\medskip

{\bf Wards Optimization Problem \cite{batagelj1988generalized}.} Let $X$ be a data set with a  dissimilarity measure $d$ and let $k \in \N$. Find a splitting of $X$ into $k$ pairwise disjoint sets $Y_1,\ldots,Y_k$ which minimizes the generalized squared error function:
\begin{equation} \label{wardsA}
\EW(Y_1,\ldots,Y_k):=\sum \limits_{i=1}^k \ss(Y_i),
\end{equation}
where $\ss(\cdot)$ is defined by \eqref{defWard}.

\medskip

%%%%%%%%%%%%%%%%%%%%%%%%%%%%%%%%
\subsection{Spherical Wards criterion function}

The Cross-Entropy Clustering ({\CEC}) is a kind of distribution-based clustering which divides an Euclidean data set into groups such that each group is described by optimally fitted Gaussian probability distribution \cite{tabor2013}. The effects of the clustering are similar to those obtained by GMM, but the optimizing criterion function is different. Its value determines the statistical code length of memorization of an arbitrary element of a data set in the case when each cluster uses its own coding algorithm. In particular, the introducing of one more cluster (coding algorithm) requires an additional cost of its identification (increase of the entropy). In consequence, the maintaining of too many clusters is not optimal and it allows for the automatic reduction of unnecessary groups. Another advantage of {\CEC} is that the clustering is performed in a comparable time to computationally efficient k-means method. For more details the reader is referred to \cite{tabor2013, smieja2013image, tabor2013detection}.

Spherical Cross-Entropy Clustering ({\sCEC}) is a variant of {\CEC} which takes into account the family of spherical Gaussian distributions. Since for every group the optimal spherical Gaussian distribution is matched, then data set is partitioned into spherically-shaped clusters. For a splitting $Y_1,\ldots,Y_k$ of $X$ the associated criterion function is defined by \cite{tabor2013}
\begin{equation} \label{spher1}
\begin{array}{l}
\frac{N}{2} \ln(\frac{2 \pi e}{N}) + \sum\limits_{i=1}^k \frac{|Y_i|}{|X|} \cdot \left[ -\ln \frac{|Y_i|}{|X|} + \frac{N}{2} \ln\left( \frac{|X|}{|Y_i|} \tr(\Sigma_{Y_i}) \right) \right],
\end{array}
\end{equation}
where $\Sigma_Y$ is a covariance matrix of group $Y$ and $\tr(\Sigma_Y)$ is a trace of $\Sigma_Y$.

Let us first observe that the notion of covariance matrix can be easily removed from the expression \eqref{spher1}.
\begin{proposition} \label{prop:tr}
If $Y \subset \R^N$ then \cite{tabor2013}:
$$
\tr(\Sigma_Y) = \ss(Y).
$$
\end{proposition}
In consequence the application of Wards approach \eqref{defWard} facilitates its interpretation in non Euclidean case for a fixed $N > 0$.

For fully explanation of the formula \eqref{spher1} in the context of non Euclidean space, the value of dimension $N$ has to be specified. As the most reasonable way to set this value we recommend to use the estimation of a dimension of $X$. In the present study we apply the Maximum Likelihood Estimation (MLE) of intrinsic dimension of $X$ proposed in \cite{maxdim} and modified in \cite{comments}. More precisely, given $X = \{x_1,\ldots,x_n\}$ the maximum likelihood estimator of a dimension $N$ of $X$ calculated for each $x \in X$ equals \cite{maxdim}:
$$
\hat{N}_k(x) = \frac{1}{k-1} \sum_{j=1}^{k-1} \log \frac{d(x, x_k)}{d(x, x_j)},
$$
for $k \in \{1,\ldots,n\}$. Since the above value is dependent on the choice of $k$ and $x$, then one should average the results over $x \in X$ and $\tilde{K} \subset \{1,\ldots,n\}$ to obtain the final estimator of $N$ \cite{comments}.

Nevertheless, one can tune this value in the learning process as well as may set it to any positive number. In the experimental section we show that for high values of $N$  more clusters are created in the clustering while for low values of $N$ the method prefers to reduce a number of groups. From now on, $N$ will be treated as a free parameter selected by the user, but we keep in mind that the easiest way to tune this value is to use the MLE procedure described above.

All in all, the generalized Wards approach and the appropriate choice of the dimension parameter $N$ allow for the understanding of spherical cross-entropy criterion function in arbitrary data set with a dissimilarity measure. In consequence, the informal notion of spherical Gaussian probability distribution based on any dissimilarity measure could be considered. We conclude this subsection with a formulation of spherical Wards ({\sWards}) optimization problem:

\medskip

{\bf Spherical Wards Optimization Problem. } Let $X$ be a data set with a dissimilarity measure $d$, $n \in \N$ be an initial number of clusters and $N > 0$ be a free parameter. Find $k \leq n$ and a partition $Y_1,\ldots,Y_k$ of $X$ which minimizes spherical Wards criterion function 
\begin{equation} \label{sWards11}
\begin{array}{l}
\ESW(Y_1,\ldots,Y_k; N):=\\
\frac{N}{2} \ln(\frac{2 \pi e}{N}) + \sum\limits_{i=1}^k \frac{|Y_i|}{|X|} \cdot \left[ \frac{N}{2} \ln(\ss(Y_i)) - \frac{N+2}{2} \ln\left(\frac{|Y_i|}{|X|}\right) \right],
\end{array}
\end{equation}
where $\ss(\cdot)$ is defined by \eqref{defWard}.

\medskip

\subsection{Clustering algorithm}

One can show that the natural modification of the Hartigan algorithm \cite{batagelj1988generalized, hartigan1979algorithm, tabor2013} can be used to minimize the {\sWards} criterion function \eqref{sWards11}. We will now discuss its technical aspects.

The procedure can be divided into two parts: initialization and iteration. In the initialization phase $n \in \N$ groups are created randomly. During iteration the algorithm reassigns elements between clusters in order to minimize the {\sWards} criterion function \eqref{sWards11}.

More precisely, in the iteration part we repeatedly go over all elements of $X$ applying the following steps: 
\begin{enumerate}
%\item If $x \in X$ is unassigned then assign it to this valid group for which the increase of energy \eqref{sWards11} is minimal,
\item Reassign $x \in X$ to this cluster for which the decrease of energy \eqref{sWards11} is maximal,
\item If a probability of some cluster is less than a fixed number $\e>0$, then remove this cluster and assign its elements to these groups for which the increase of energy \eqref{sWards11} is minimal,
\end{enumerate}
until no group membership has been changed. 

The number $\e$ was introduced to speed up the reduction of redundant clusters. In our experiments we always use the value $\e=1\%$. Thus, the group is removed if it contains less than $1\%$ of all elements of $X$. Clearly, the procedure is not deterministic and leads to a local minimum of \eqref{sWards11} \cite{jain1999}. Therefore, to provide the satisfactory results the algorithm should be evaluated several times -- the final result is that which gives the minimal value of {\sWards} criterion function.

The above algorithm can be seen as an online version of standard partitional clustering procedure which is able to reduce unnecessary groups. Every time the element is processed the clusters parameters are recalculated. This implies that to efficiently apply this procedure we have to recompute
$$
\ss(Y \cup \{x\}) \mbox{ and }\ss(Y \setminus \{x\}).
$$
For this purpose the following formulas are useful:
\begin{proposition} \label{HartCor}  \cite{spath1975cluster}
Let  $Y \subset X$ and $x \in X$.

a) If $x \not\in Y$, then
$$
\ss(Y \cup \{x\}) = \frac{|Y|}{|Y|+1}\ss(Y) + \frac{1}{|Y|+1}\D{\{x\}, Y}.
$$

b) If $x \in Y$, then
$$
\ss(Y \setminus \{x\})=\frac{|Y|}{|Y|-1}\ss(Y) - \frac{1}{|Y|-1}\D{\{x\}, Y}.
$$
\end{proposition}

Given $k$ clusters, the computational complexity of one iteration of standard Hartigan procedure requires about $k \cdot N \cdot |X|$ operations (for data sets contained in $\R^N$). When applying the Wards approach this complexity changes to $k \cdot |X|^2$ operations. Since the mean of cluster is not defined in general situation, one has to pay an additional cost of recalculating the within cluster sum of squares during every reassigning. However, we do not need to recalculate the distance between the reassigning elements and the mean of a cluster which decreases the computational cost $N$ times.

\section{Generalized Voronoi diagram}

There arises a natural problem how to graphically present the clustering results. Clearly, we can mark the elements of each cluster with different label. However, in practice it is usually more clear to show the division of the whole space. In this section we show that we can naturally obtain an equivalence of the Voronoi diagram for any criterion function in non Euclidean space. In particular we apply these results to define the Voronoi diagram for {\sWards}.

\subsection{Classical diagram}

Let us recall that in the case of classical version of Voronoi diagram ($k$-means method) the point $x$ is associated with this cluster whose center is the closest to $x$. More precisely, it is classified to this cluster $Y_i$ which minimizes $d(x; \m_i)$, where $m_i$ is a mean of $Y_i$. We would like to mention that one can consider the alternative to the Voronoi diagrams as described in \cite{telgarsky2010hartigan}. It provides the partition of data but does not induce a natural partition of the space (see \cite{telgarsky2010hartigan} for more details).

To generalize the notion of the Voronoi diagram to non Euclidean space (Wards k-means), we need to be able to compute the distance of a point from the center of the cluster (without using it in the computations).

\begin{proposition} \label{th:2} \cite{spath1975cluster}
Let $x \in \R^N$ be fixed and $Y \subset \R^N$ be a subset  of $\R^N$ with mean $\m_{Y}$. Then
$$
\|x-\m_Y\|^2=\frac{1}{|Y|}\sum\limits_{y \in Y} \|x-y\|^2 - \frac{1}{2|Y|^2}
\sum \limits_{y \in Y} \sum \limits_{z \in Y} \! \! \|y-z\|^2.
$$
\end{proposition}
The above allows the formulation of the analogue of the square of the ``classical'' distance of a point $x$ from the center of $Y$. Let $Y$ be a subset of data space $X$ with a dissimilarity measure $d$ and let $x \in X$ be fixed. We define the mean square distance of $x$ from $Y$ by
\begin{equation}\label{voro}
d^2(x;Y):=\frac{1}{|Y|}(\D{\{x\},Y}-\ss(Y)).
\end{equation}
Applying the above formula one can draw the equivalence of the Voronoi diagram for Wards k-means, i.e. an element $x \in X$ is classified to this cluster which minimizes \eqref{voro}.

\subsection{Diagram for arbitrary criterion function}

We are now going to present a reasoning which allows to create a kind of Voronoi diagram for arbitrary criterion function. This will be useful for constructing a division of the space for the case of {\sWards} method. Obtained results are consistent with the classical Voronoi diagram in the case of Wards $k$-means presented in previous section.

Let $X$ be a space with a dissimilarity measure $d$ and let $Y \subset X$ represent our data. We extend $X$ by introducing a weight function 
$$
w: X \ni x \to \left\{
\begin{array}{ll}
	w(x) \in [0,+\infty) &, x \in Y, \\ 
	0 &, x \in X \setminus Y,
\end{array}
\right.
$$
which assigns a weight to every element of $X$. Then we consider an extended data set
$$
Y^w = \{(y,w(y)): y \in Y\}.
$$

%rozszerzone definicje
We define the operations $\D{\cdot,\cdot}$ and $\ss(\cdot)$ adapted for $Y^w$. Given $Z, Y_1, Y_2 \subset Y$ we put:
\begin{enumerate}
\item $|Z^w| := \sum\limits_{z \in Z} w(z)$,
\item $\D{Y_1^w,Y_2^w} := \sum\limits_{y_1 \in Y_1} \sum\limits_{y_2 \in Y_2} d^2(y_1,y_2) w(y_1)w(y_2)$,
\item $\ss(Z^w) := \frac{1}{2|Z^w|} \D{Z^w,Z^w}$.
\end{enumerate}
Then the analogue of k-means criterion function equals:
\begin{equation} \label{wardsW}
\EW(Y^w_1,\ldots,Y^w_k) = \sum_{i=1}^k \ss(Y^w_i),
\end{equation}
where $Y_1,\ldots,Y_k$ is a splitting of $Y$. If $w_{| Y} \equiv 1$ then \eqref{wardsW} coincides with \eqref{wardsA}.

In order to explain our technique assume that $Y_1,\ldots,Y_k$ is a splitting of data set $Y$ and $E$ is an arbitrary criterion function. For a fixed point $x \in X$ we consider a mapping 
$$
\begin{array}{l}
E^i_{x, [Y^w_1,\ldots,Y^w_k]}:h \to \\[1.2ex]
E(Y^w_1,\ldots,Y^w_{i-1},(Y_i \cup \{x\})^{w+h\delta_x},Y^w_{i+1},\ldots,Y^w_k),
\end{array}
$$
where $h \geq 0$ and $i \in \{1,\ldots,k\}$. It determines the value of criterion function $E$ when $x \in X$ is associated with $i$-th cluster with a weight increased by $h$.

We define the functions (wherever they exist)
$$
\partial_i E(x,[Y^w_1,\ldots,Y^w_k]):=(E^i_{x,[Y^w_1,\ldots,Y^w_k]})'(0),
$$
for $i \in \{1,\ldots,k\}$. Observe that $\partial_iE$ coincides with the infinitesimal change in energy
when we add $x$ to the $i$-th cluster. Thus, in Voronoi diagram the point $x \in X$ should be assigned to this cluster which minimizes $\partial_iE(x,[Y^w_1,\ldots,Y^w_k])$.

Let us show that the above reasoning is consistent with the classical results \eqref{voro} for Wards $k$-means criterion function \eqref{wardsW}:
\begin{theorem} \label{latweTw}
Let $Y$ be a subset of a space $X$ with a dissimilarity measure $d$ and let $w(y) = 1$, for all $y \in Y$, be a weight function. If $E$ denotes the squared error function \eqref{wardsW} and $Y_1,\ldots,Y_k$ is a fixed splitting of $Y$ then
$$
\partial_iE(x,[Y^w_1,\ldots,Y^w_k])=d^2(x;Y_i),
$$
for $x \in X$ and $i \in \{1,\ldots,k\}$.
\end{theorem}

\begin{proof}
Let $h > 0$. By Corollary \ref{HartCor}, we have
$$
\begin{array}{l}
\frac{1}{h}[E(Y^w_1,\ldots,Y^w_{i-1},Y^w_i\cup \{(x,h)\},Y^w_{i+1},\ldots,Y^w_k) \\[1.2ex]
\,\,\,\,\,\,\, -E(Y^w_1,\ldots,Y^w_k)] \\[1.2ex]
=\frac{1}{h}\left[\ss((Y_i\cup\{x\})^{w + h\delta_x}) -\ss(Y^w_i)\right] \\[1.2ex]
=\frac{1}{h}\left[\frac{|Y_i^w| \ss(Y_i^w) + \D{\{(x,h)\}, Y_i^w}}{|Y_i^w|+h} - \ss(Y_i^w)\right] \\[1.2ex]
=\frac{1}{h}\frac{|Y_i^w| \ss(Y_i^w) + h \D{\{(x,1)\}, Y_i^w} - (|Y_i^w| + h)\ss(Y_i^w)}{|Y_i^w|+h}  \\[1.2ex]
=\frac{\D{(x,1),Y_i^w}-\ss(Y_i^w)}{|Y_i^w|+h}.
\end{array}
$$
Since $w_{|Y} \equiv 1$ then
$$
\begin{array}{l}
\frac{\D{(x,1),Y_i^w}-\ss(Y_i^w)}{|Y_i^w|+h} =\frac{\D{x,Y_i}-\ss(Y_i)}{|Y_i^w|+h}\to \\[1.2ex]
\frac{1}{|Y_i|} (\D{x,Y_i}-\ss(Y_i)) \text{ , as } h \to 0,
\end{array}
$$
which yields the assertion of the theorem.
\end{proof}

\subsection{Voronoi diagram for {\sWards}}

The following theorem presents how to create the Voronoi diagram for {\sWards} criterion function:
\begin{theorem}
Let $Y$ be a subset of a space $X$ with a dissimilarity measure $d$ and let $w(y) = 1$, for all $y \in Y$, be a weight function. If $E$ denotes the {\sWards} criterion function for a data set with weights and $Y_1,\ldots,Y_k$ is a fixed splitting of $Y$ then
$$
\begin{array}{l}
\partial_iE(x,[Y^w_1,\ldots,Y^w_k])\\[1.2ex]
=\frac{1}{|X|}\left[\frac{N}{2}\left(\ln(\ss(Y_i))+|Y_i|\frac{d^2(x;Y_i)}{\ss(Y_i)}\right)- \frac{N+2}{2} (\ln|Y_i|+1)\right] .
\end{array}
$$
\end{theorem}

\begin{proof}
Roughly speaking, Theorem \ref{latweTw} says that $\partial_i \ss(Y^w_i)=d^2(x;Y_i)$. Moreover, $\partial_i|Y^w_i|=1$. Applying the operator $\partial_i$ and the above to  \eqref{sWards11} we easily get the assertion of the theorem.
\end{proof}

\begin{figure}[t]
\centering
\includegraphics[width=2.0in]{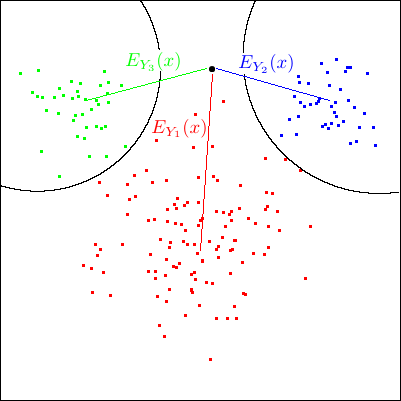}
	\caption{{\bf Construction of Voronoi diagram.} Given a partition of $Y \subset X$, the procedure iterates over all data space elements $x \in X$ (including also elements which did not participate in the clustering), calculates the values of assignment function $E_{Y_i}(x)$ for each cluster $Y_i$ and attaches $x$ to this group $Y_j$ which minimizes $E_{Y_j}(x)$.} 
	\label{fig:vorr} 
\end{figure}

Consequently, given a partition $Y_1,\ldots,Y_k$ of $Y$, to associate a point $x \in X$ to a cluster it is sufficient to find $i \in \{1,\ldots,k\}$ which minimizes
\begin{equation}\label{vorInt}
E_{Y_i}(x) = \ln(\ss(Y_i))+|Y_i|\frac{d^2(x;Y_i)}{\ss(Y_i)}-(1+\frac{2}{N}) \ln|Y_i|.
\end{equation}
If $X$ is infinite, then one can apply its quantization into a finite number of regions before applying a Voronoi diagram. The reader is referred to Figure \ref{fig:vorr} for more detailed explanation of the above described procedure.

\section{Experiments}

In this section we discuss some fundamental properties as well as the potential applications of proposed clustering method and present a short evaluation study. The implementation of {\sWards} is available from http://www.ii.uj.edu.pl/{\textasciitilde}smieja/sWards-app.zip\footnote{Contact the first author for the explanations.}.

\subsection{Synthetic data sets}

In order to show the capabilities of {\sWards} we examined its resistance on the change of scale and its sensitivity on the unbalanced data. We compared the clustering results with the ones obtained with use of related methods which can be applied for non Euclidean spaces: Wards k-means and Spectral Clustering (kernlab R package was used for the implementations of this algorithm \cite{karatzoglou2004kernlab}). Since {\sWards} automatically detects the resultant number of groups, then we ran it with 10 initial clusters while the other methods used the number of groups returned by {\sWards}\footnote{Such a technique for a detection of clusters number was chosen in order to provide the correspondence between clustering results for all methods.}. The value of parameter $N$ (dimension of space) for {\sWards} was set automatically with use of MLE method \cite{maxdim, comments}. To provide more stable results, each algorithm was run 10 times and the result with the lowest value of criterion function was chosen.

\begin{example}\label{scaleInvariance}{\bf Scale invariance} 

In the first experiment we examined the invariance of algorithms on the change of scale. A data set was generated from the mixture of two spherical Gaussian distributions,
$$
\frac{1}{2} \G_1(r) + \frac{1}{2} \G_2(1-r)
$$
with different covariance matrices
$$
C_1 = 
\begin{pmatrix}
r & 0\\
0 & r
\end{pmatrix}
\text{ , } C_2 = 
\begin{pmatrix}
1-r & 0\\
0 & 1-r
\end{pmatrix}
\text{, for } r \in (0,1),
$$
centered at 
$$
\m_1 = (-1,0) \text{ , } \m_2 =(1,0).
$$
The parameter $r$ controls the width of Gaussians.

The Figure \ref{fig:proportion2} presents the ratios of resulted clusters sizes. The {\sWards} method is robust to the change of scale -- the clusters remained almost equally-sized for all $r \in (0,1)$. The clustering result was the most dependent on the widths of Gaussians in the case of k-means.

\end{example}

\begin{figure}[t]
\centering
\includegraphics[trim=0cm 0.6cm 0cm 2cm, clip=true, width=3.0in]{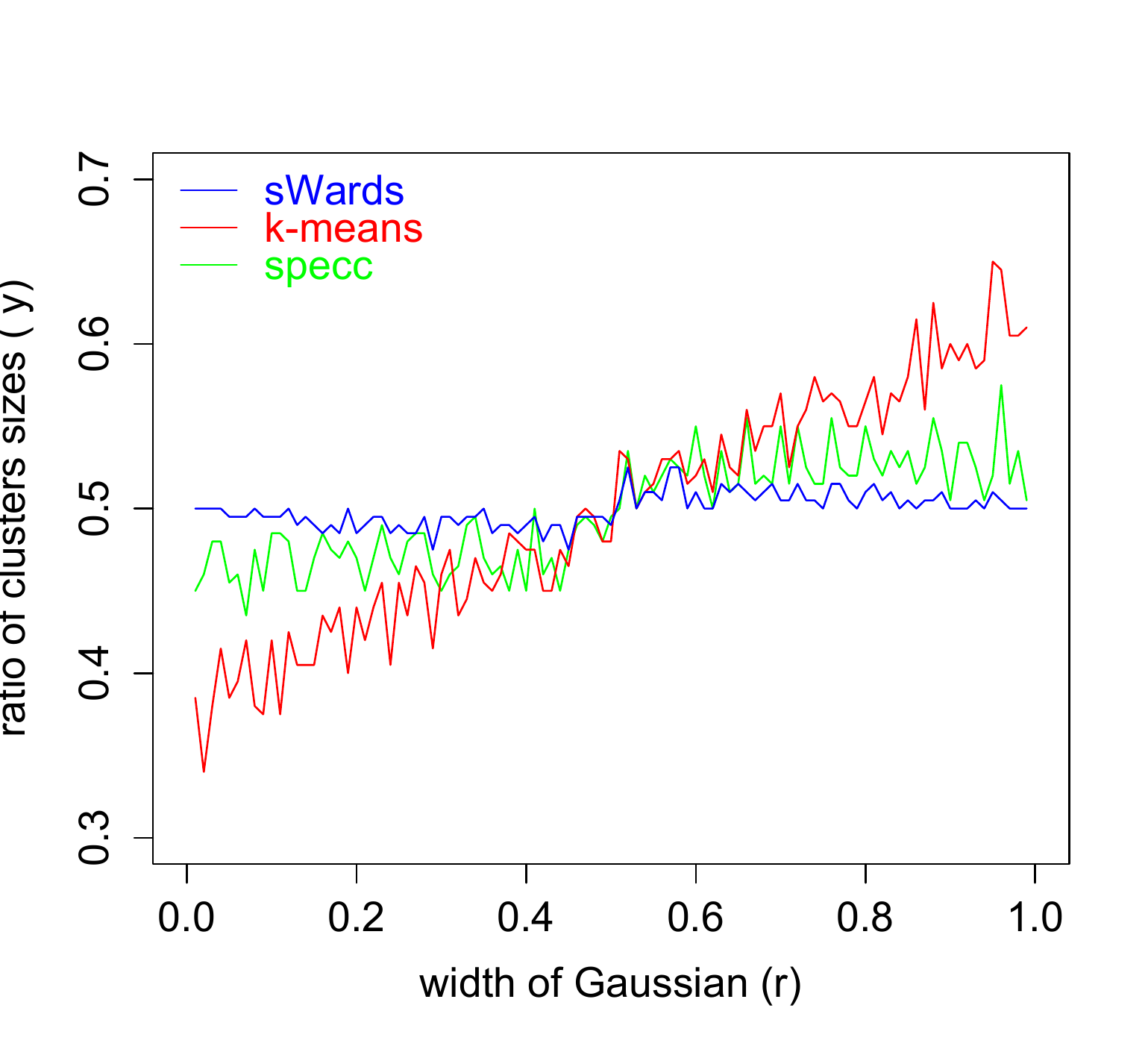}
	\caption{{\bf Scale invariance.} The rations of clusters sizes for a data set generated from the mixture of two spherical Gaussian distributions $\frac{1}{2}~\G_1~(~r~)~+~\frac{1}{2}~\G_2~(~1~-~r~)$ when changing the width $r$ of Gaussians. The optimal curve should be a constant function, $y=\frac{1}{2}$.} 
	\label{fig:proportion2} 
\end{figure}

\begin{figure}[t]
\centering
\includegraphics[trim=0cm 0.6cm 0cm 2cm, clip=true, width=3.0in]{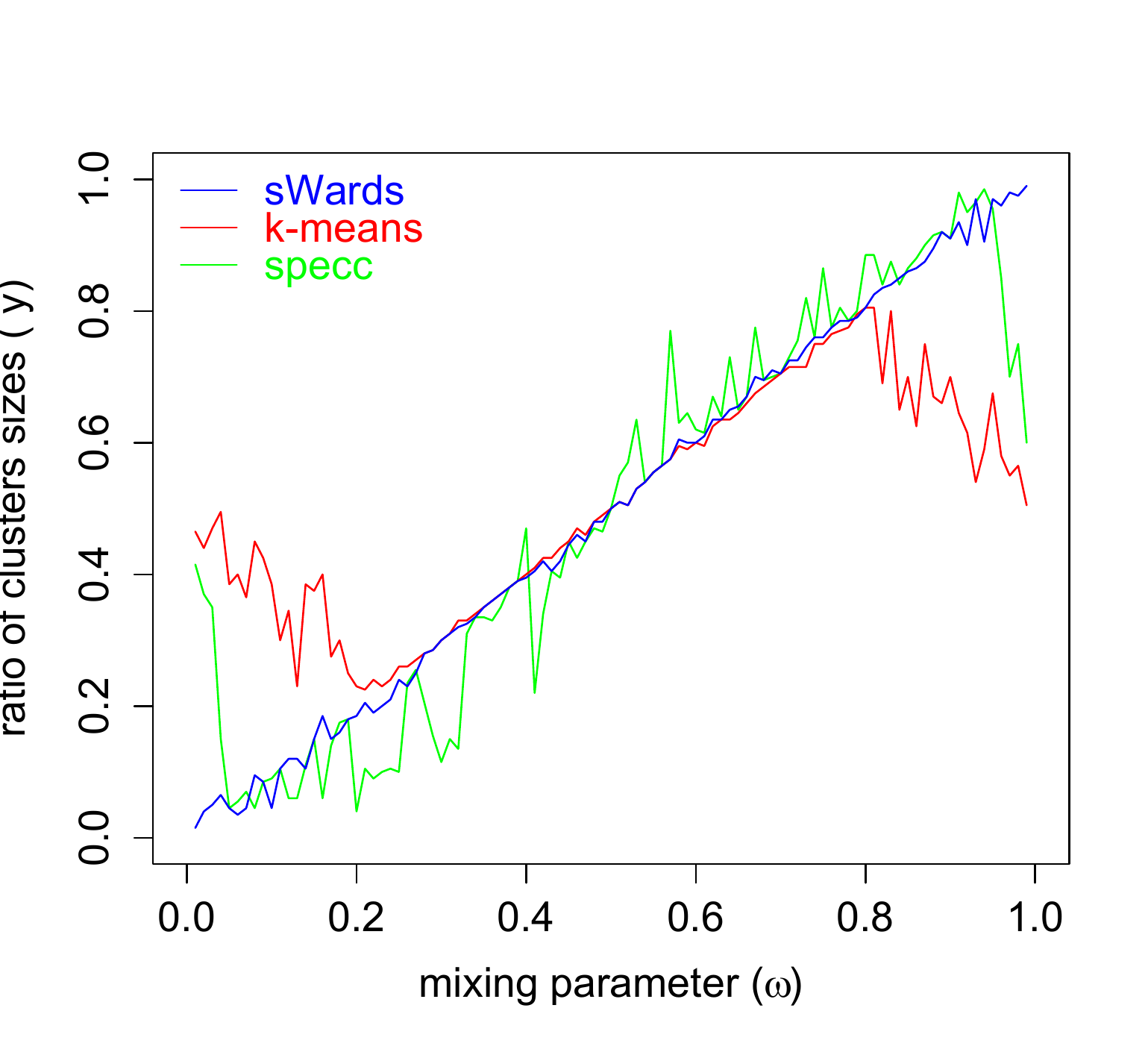}
	\caption{{\bf Sensitivity on the unbalanced data.} The rations of clusters sizes for a data set generated from the mixture of two spherical Gaussian distributions $\omega~\G_1~+~(~1~-~\omega)~\G_2$. Number $\omega \in (0,1)$ controls the number of elements produced by each Gaussian. The optimal curve should be a linear function, $y	=\omega$.}
	\label{fig:proportion} 
\end{figure}

\begin{example}\label{unbalancedData} {\bf Unbalanced data}

We have also tested how the number of elements generated from the individual distributions affects the clustering results. For this purpose data was generated from the mixture of two Gaussians
$$
\omega \cdot \G_1 + (1-\omega) \cdot \G_2 \text{ , for } \omega \in (0,1),
$$
with identical covariance matrices
$$
C_1 = C_2 =  
\begin{pmatrix}
\frac{1}{2} & 0\\
0 & \frac{1}{2}
\end{pmatrix},
$$
but different centers
$$
\m_1 = (-1,0) \text{ , } \m_2 = (1,0).
$$
The number of elements generated from each Gaussian is determined by the value of parameter $\omega$.

The ratios of clusters sizes are shown in the Figure \ref{fig:proportion}. One can observe that the proportions specified by $\omega$ was preserved by {\sWards} method. In the Spectral Clustering the  results are less stable. On the other hand Wards k-means has a tendency to build equally-sized clusters.

\end{example}

\subsection{Dimension estimation}

To apply the {\sWards} criterion function in the case of arbitrary non Euclidean space the value of dimension parameter $N$  needs to be specified. In the previous subsection we showed that the reasonable clustering results can be obtained calculating this value using MLE method \cite{maxdim, comments}. We will experimentally show how the clustering effects differ when the value of $N$ changes.

\begin{figure}[t]
\centering
\includegraphics[trim=0cm 0.6cm 0cm 1.8cm, clip=true, width=3.0in]{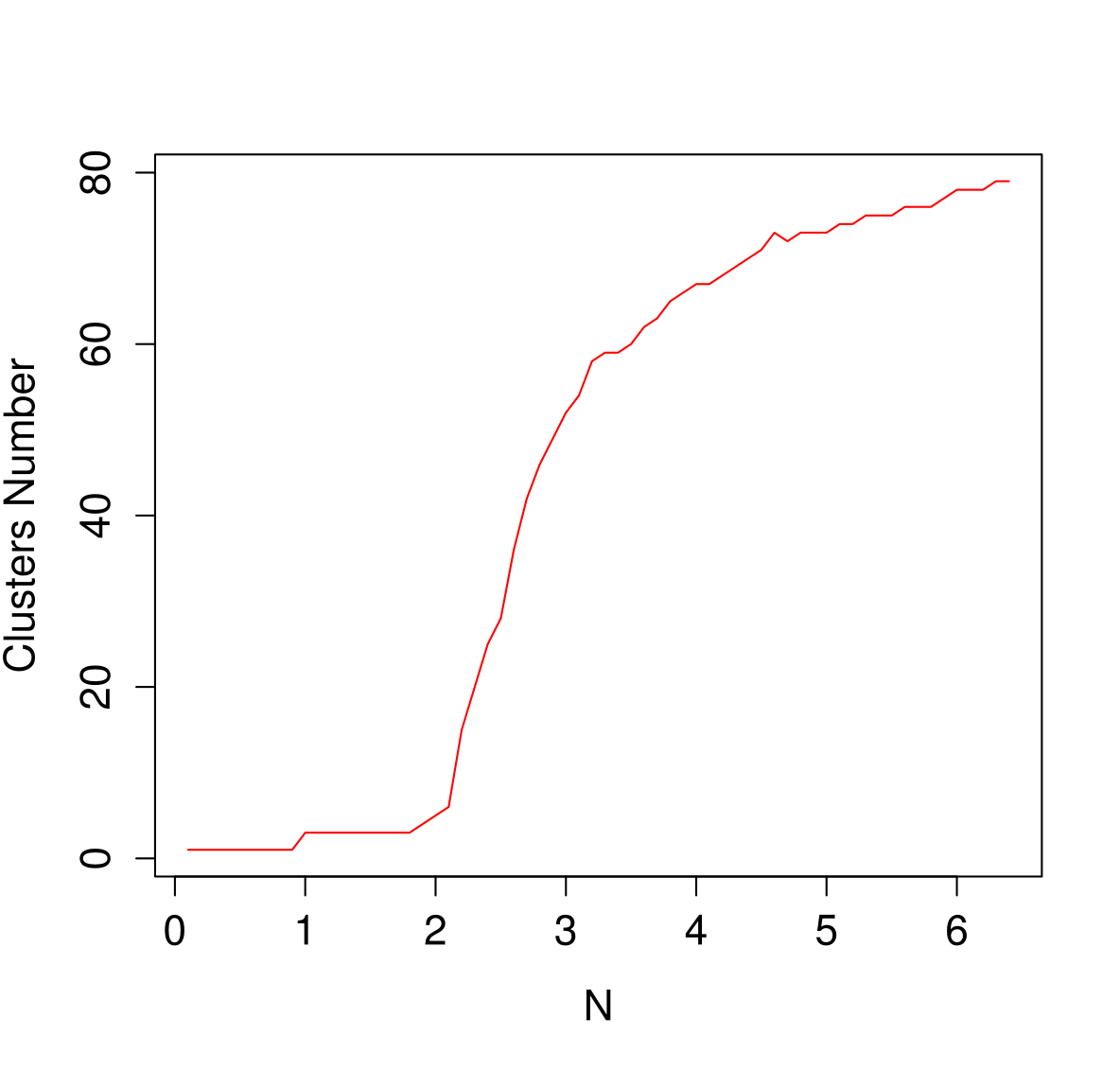}
	\caption{{\bf Clusters detection.} The influence of the value of parameter $N$ on the resulted number of clusters. Maximal number of clusters was set to 100.}
	\label{fig:dimCl} 
\end{figure}

\begin{figure}[t]
\centering
\subfigure[Wards k-means clustering with $k=3$.]{\label{env2:a}\includegraphics[width=1.55in]{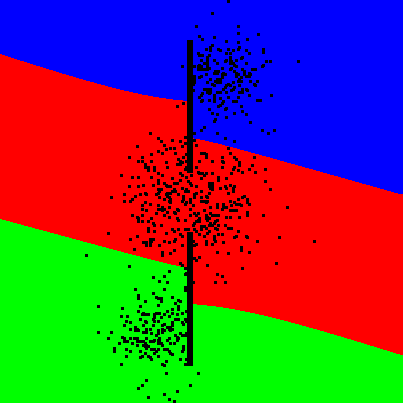}} \quad
\subfigure[{\sWards} clustering started with 10 initial clusters.]{\label{env2:b}\includegraphics[width=1.55in]{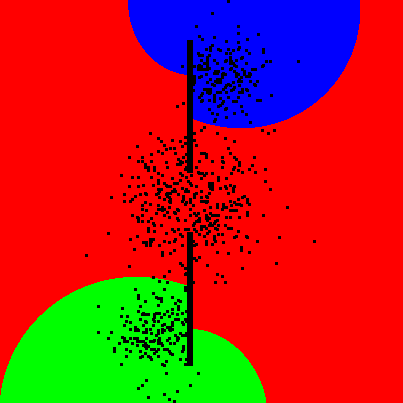}} \quad
	\caption{{\bf Populations districts on the space with barriers.} Voronoi diagrams constructed by Wards k-means and {\sWards} on a data space with barriers containing three populations generated from random walk models.}
	\label{fig:env2} 
\end{figure}

\begin{figure}[t]
\centering
\includegraphics[trim=0cm 0.6cm 0cm 1.8cm, clip=true, width=3.0in]{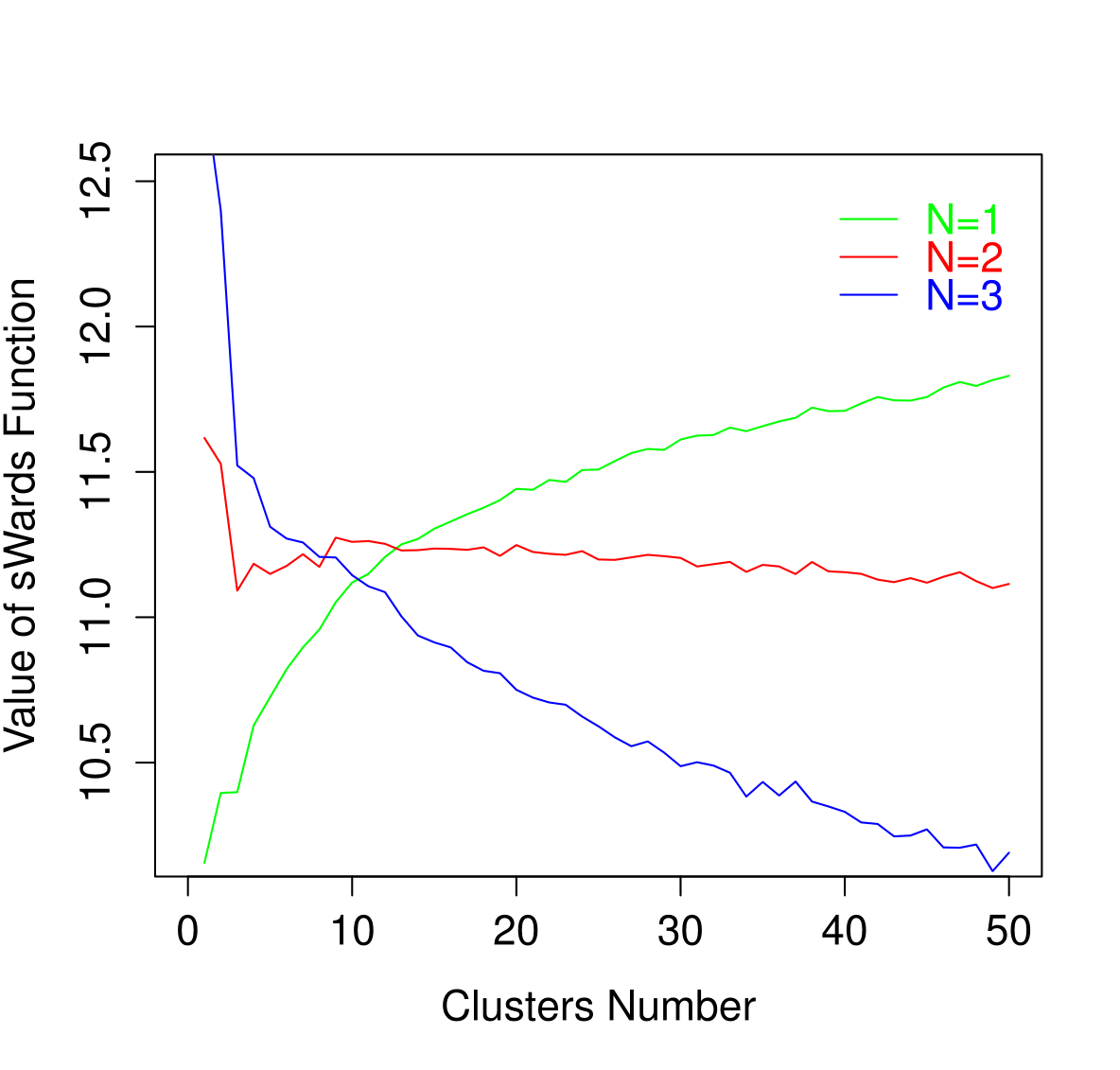}
	\caption{{\bf Shape of criterion function.} The influence of the clusters number on the shape of {\sWards} criterion function for three exemplary values of $N$.}
	\label{fig:clCost} 
\end{figure}

\begin{figure}[t]
\centering
\subfigure[Wards k-means clustering with $k=2$.]{\label{env1:a}\includegraphics[width=1.55in]{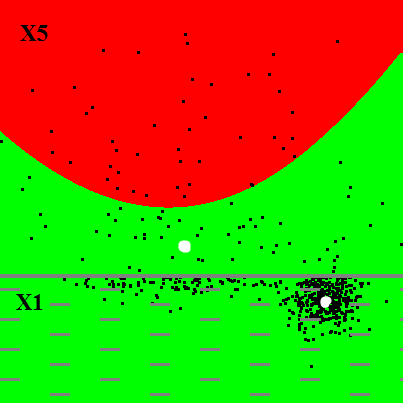}} \quad
\subfigure[{\sWards} clustering started with 10 initial clusters.]{\label{env1:b}\includegraphics[width=1.55in]{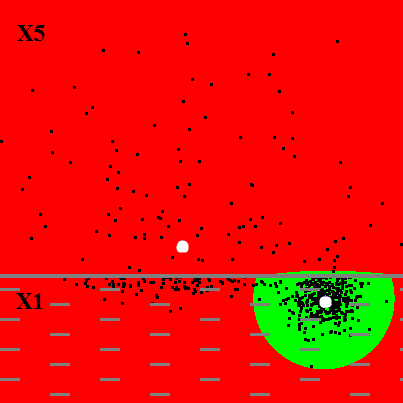}} \quad
	\caption{{\bf Populations districts on the space with regions.} Voronoi diagrams constructed by Wards k-means and {\sWards} on a data space divided into two regions $X_1$ and $X_5$ containing three populations generated from random walk models. The speed of movements in $X_5$ is 5 times higher than in $X_1$.}
	\label{fig:env1} 
\end{figure}

\begin{table*}[t]
\caption{{\bf UCI evaluation. }Comparison of clustering results (measured by Rand index) of UCI data sets between {\sWards}, Wards k-means and Spectral Clustering for Euclidean and RBF dissimilarities. The estimated numbers of clusters (Est. Cl.) by {\sWards} were used for other algorithms. The MLE was applied for setting the parameter $N$.}
\label{UCI}
\vskip 0.15in
\begin{center}
\begin{small}
\begin{sc}
\begin{tabular}{lcc|cccc|cccr}
\hline
%\abovespace\belowspace
 & & & \multicolumn{4}{c|}{Euclidean dissimilarity} & \multicolumn{4}{c}{RBF dissimilarity}\\
Data & True Cl. & $N$ & Est. cl. & {\sWards} & k-means & Specc & Est. cl. & {\sWards} & k-means & Specc\\
\hline
%\abovespace
Cmc & 3 & 2.64 & 4 & {\bf 0.61} & 0.57 & 0.58 & 2 & {\bf 0.55} & 0.51 & 0.51 \\
Ecoli & 8 & 3.72 & 9 & {\bf0.88} & 0.83 & 0.77 & 10 & {\bf0.84} & 0.79 & 0.8 \\
Glass & 7 & 3.07 & 8  & {\bf0.71} & 0.7 & 0.68 & 7  & 0.7 & {\bf0.71} & {\bf0.71} \\
Hayes-r. & 3 & 1.85 & 5 & {\bf 0.62} & 0.58 & 0.59 & 5 & {\bf 0.61} & 0.5 & 0.6 \\
Ionosph. & 2 & 5.03 & 4 & 0.55 & 0.52 & {\bf 0.61} & 4 & 0.57 & {\bf 0.61} & 0.58 \\
Iris & 3 & 2.49 & 4 & {\bf0.85} & 0.81 & 0.83 & 5 & {\bf0.85} & 0.84 & 0.83 \\
Tae & 3 & 2.06 & 6 & {\bf 0.61} & {\bf 0.61} & 0.6 & 5 & {\bf0.62} & 0.58 & 0.6 \\
Wine &3 & 1.64 & 4 & {\bf 0.75} & 0.63 & 0.68 & 5 & {\bf 0.58 }& 0.55 & 0.55\\
%\belowspace
Yeast & 10 & 4.81 & 11 & 0.64 & {\bf0.73} & {\bf0.73} & 10 & 0.63 & {\bf0.73} & {\bf0.73}\\
\hline
\end{tabular}
\end{sc}
\end{small}
\end{center}
\vskip -0.1in
\end{table*}

\begin{example} {\bf Clusters number detection}

Let us first examine the impact of the value of parameter $N$ on the detection of the resultant number of groups. For this purpose a mouse-like set (see Figure \ref{fig:mouse}) was clustered with different values of $N$ starting from 100 initial groups. The resultant number of groups are illustrated in the Figure \ref{fig:dimCl}.

The immediate observation is that the increase of the value of $N$ results in the increase of the detected number of groups. One can observe that for $N < 1$ the entire data set was recognized as one group. For $N \in (1,2)$ the mouse-like set was partitioned into three groups which seems to be the most appropriate partitioning. For $N > 2$ the number of groups began to grow rapidly.
\end{example}

\begin{example} {\bf Shape of criterion function}

To get more insight into the influence of dimension parameter on the discovered number of clusters, we analyzed the shape of {\sWards} criterion function for different values of $N$. Since the {\sWards} automatically reduces unnecessary clusters, it is not possible to directly specify the number of groups. Therefore, a mouse-like data set (see Figure \ref{fig:mouse}) was first partitioned into expected number of groups with use of k-means. Then, the {\sWards} criterion function was calculated for each partition. 

It is clear from Figure \ref{fig:clCost} that the criterion function yields a global minimum for 3 clusters when $N=2$. Therefore, in most cases the algorithm ends with 3 groups. For $N=1$ the cost of maintaining clusters increases and the algorithm generally includes all elements into one group. The function is decreasing for $N=3$. It means that the method rarely reduces clusters. The last case can be a very useful variant of {\sWards} when the resulting number of groups should not be discovered by the algorithm but specified directly by the user.
\end{example}

\subsection{Applications}

In this section we show that the proposed method is very useful in the analysis of biological models of populations. It is assumed that a population follows a random walk model $P(x,n,t)$ on a plane \cite{codling2008random}, where at each unit of time an instance moves randomly in one of four directions: left, right, up or down. More precisely, given a starting point (seed) $x \in X$, $n$-instances are generated from a random walk model assuming $t$ time units. It is worth to mention that a probability distribution of a population converges to spherical Gaussian one \cite{codling2008random}. Given a data set consisting of $k$ populations we would like to discover them during a clustering process. Constructed Voronoi diagram determines the corresponding populations districts in the whole space.

Let us observe that, in practice the environment does not represent an Euclidean space. Indeed, a plane is usually crossed by rivers and barriers. Moreover, the environment can be divided into various regions, e.g. meadows, seas, forests etc., which changes the speed of movement of individuals. These modifications change the classical Euclidean metric -- the distance between elements has to take into account all the aforementioned circumstances. In the experiments we analyze two cases of populations environments.

\begin{example} \label{biloEx1} {\bf Environment with barriers.}
Let us consider three populations living in the environment showed in Figure \ref{fig:env2} crossed by two barriers which modify Euclidean distance function. Basically, the distance between elements located on the opposite sides of barrier is calculated as a shortest path which does not cross the barrier. 

Regions occupied by populations can be obtained with use of Voronoi diagram. Is is clear from Figure \ref{env2:a} that Wards k-means discovered populations districts as horizontal stripes which is not an appropriate model. More accurate partition results from {\sWards} (see Figure \ref{env2:b}), where detected regions form circular shapes. Partitions agreement measured by Rand index \cite{rand1971objective} for Wards k-means equals $96\%$ while for {\sWards} is $98\%$.
\end{example}

\begin{example} \label{biloEx2} {\bf Environment with regions.}
In the second example let us assume that a data space $X$ is divided into two regions $X_1$ and $X_5$. In $X_5$ the individuals moves 5 times faster than in $X_1$. This inducts a dissimilarity measure on $X$ by:
$$
d(x,y):=\left\{
\begin{array}{l}
	d_E(x,y), \,\,\,\,\,\,\,\,\,\,\,\, x,y \in X_5, \\[0.6ex]
	5 d_E(x,y), \,\,\,\,\,\,\,\,\, x,y \in X_1, \\[0.6ex] 
	\inf\{5 d_E(x,z) + d_E(z,y): \text{border point } z\}, \\[0.4ex]
	 \,\,\,\,\,\,\,\,\,\,\,\, \,\,\,\,\,\,\,\,\,\,\,\, \,\,\,\,\,\,\,\,\,\,\,\,x \in X_1, y \in X_5,
\end{array}
\right.
$$
where $d_E(\cdot,\cdot)$ denotes the Euclidean distance. We consider two populations showed in Figure \ref{fig:env1} with starting points marked with white dots in $X_1$ and $X_5$ respectively. 

One can observe in Figure \ref{env1:b} that despite the form of the above dissimilarity measure, {\sWards} detected the circular-like districts of populations very well. This result can be compared with k-means clustering (see Figure \ref{env1:a}) where a produced partition does not match populations distributions. The value of Rand index for {\sWards} equals $92\%$ while for Wards k-means is $61\%$.
\end{example}

\subsection{Evaluation}

After establishing the properties as well as demonstrating basic capabilities and potential applications of introduced method we present a short evaluation. We carried out the experiments on selected UCI data sets \cite{asuncion2007uci}. In all experiments the initial number of clusters for {\sWards} was fixed two times higher than the actual number of groups. In order to provide the correspondence between the clustering results the other examined methods assumed the number of groups returned by {\sWards} as the input clusters number.

As a measure of agreements between partitions the Rand index (RI) was used \cite{rand1971objective}.  It is defined as a ratio between pairs of true positives and false negatives, and all pairs of examples. The values close to $1$ indicate that two partitions are very similar. MLE was used to calculate the optimal value of parameter $N$. Two kinds of dissimilarity measures were considered: the Euclidean distance and the dissimilarity determined by the Gaussian radial basis function (RBF). The value of sigma for RBF was estimated as a median of the squared distances between all pairs of data set elements \cite{caputo2002appearance}.

The results presented in Table \ref{UCI} show that {\sWards} reasonably well determined the final number of groups. The advantage of our method over k-means and Spectral Clustering is the most evident for the case of Ecoli data set and Euclidean distance. The worst results were obtained for Ionosphere data set. The use of RBF similarity rarely improved the accuracy of clustering. It could be caused by the fact that it is very difficult to set the optimal value for RBF sigma parameter in particular situation.

To extend the above evaluation, in the Figure \ref{fig:dimUci} we present the clustering accuracies of UCI data sets for a wide range of dimension parameter $N \in (0,10)$. One can observe that in most cases the best results were obtained when $N$ was estimated as a dimension of data. The exceptions are Glass and Yeast data sets where a slight improvement was achieved for higher values of $N$.

\begin{figure}[t]
\centering
\includegraphics[trim=0cm 0.6cm 0cm 2cm, clip=true, width=3.0in]{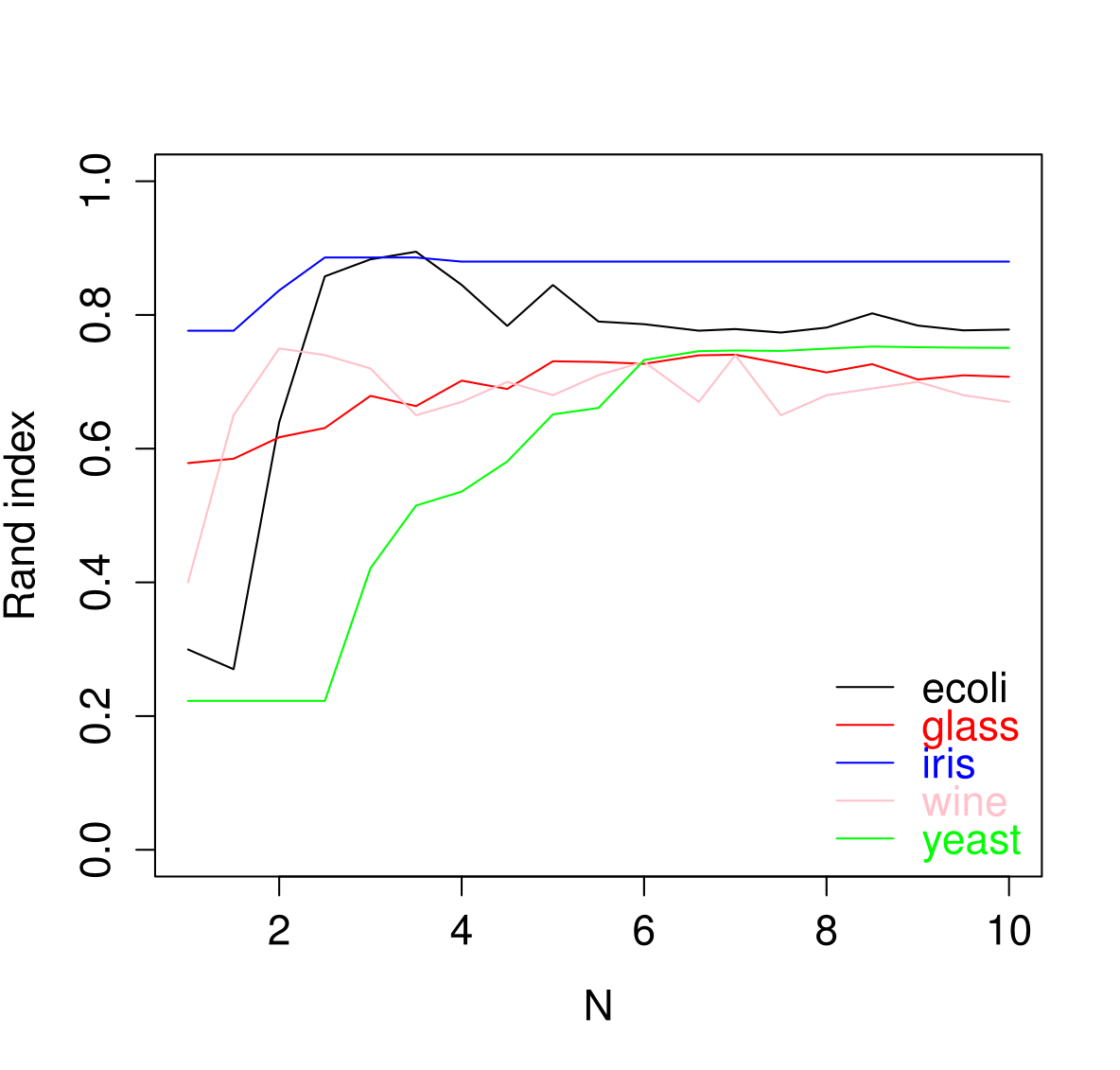}
	\caption{{\bf UCI evaluation.} The influence of the value of parameter $N$ on the clustering results of UCI data sets.}
	\label{fig:dimUci} 
\end{figure}

\section{Conclusion}

In this paper a generalization of spherical Cross-Entropy Clustering to non Euclidean spaces was presented. The proposed method uses a Wards approach to modify the cross-entropy criterion function for the case of arbitrary data sets. In consequence, obtained method allows for partitioning of non vector data into spherically-shaped clusters of arbitrary sizes. It is scale invariant technique which detects the final number of groups automatically. Our method works in comparable time to generalized Wards method while the clustering effects are similar to those produced by GMM when focusing on spherical Gaussian distributions in Euclidean spaces. 

Moreover, we generalized the notion of Voronoi diagram for the case of arbitrary criterion function based on Wards approach. This leads to identical results in the case of classical methods as k-means while it allows for formal division of data space when focusing on non Euclidean methods as {\sWards}.

\section*{Acknowledgment}

This work was partially funded by the Polish Ministry of Science and Higher Education from the budget for science in the years 2013--2015, Grant no. IP2012 055972 and by the National Science Centre (Poland), Grant No. 2014/13/B/ST6/01792.



\bibliographystyle{IEEEtran}
\bibliography{IEEEabrv,example_paper}

\end{document}